\documentclass[11pt,twocolumn]{article}

\usepackage{mathptmx}
\usepackage[T1]{fontenc}

\usepackage[margin=0.7in,columnsep=0.2in]{geometry}

\usepackage{balance}

\usepackage{titlesec}
\titleformat{\section}{\normalfont\fontsize{11}{13}\bfseries}{\thesection.}{0.5em}{}
\titleformat{\subsection}{\normalfont\fontsize{11}{13}\bfseries}{\thesubsection.}{0.5em}{}
\titleformat{\subsubsection}{\normalfont\fontsize{11}{13}\bfseries}{\thesubsubsection.}{0.5em}{}
\titlespacing*{\section}{0pt}{6pt}{2pt}
\titlespacing*{\subsection}{0pt}{5pt}{2pt}
\titlespacing*{\subsubsection}{0pt}{4pt}{1pt}

\usepackage{amsmath}
\usepackage{amssymb}
\usepackage{amsthm}
\usepackage{bm}

\usepackage{graphicx}
\usepackage{xcolor}
\usepackage{tikz}
\usepackage{pgfplots}
\pgfplotsset{compat=1.17}
\usetikzlibrary{shapes.geometric,arrows.meta,positioning,fit,backgrounds,calc,shadows.blur,decorations.pathreplacing,patterns,matrix}

\usepackage{booktabs}
\usepackage{multirow}

\usepackage{algorithm}
\usepackage{algorithmic}

\usepackage{hyperref}
\usepackage{url}
\usepackage[numbers]{natbib}

\newtheorem{theorem}{Theorem}

\makeatletter
\renewcommand{\@maketitle}{%
  \newpage
  \null
  \begin{center}%
    {\LARGE \@title \par}%
    \vskip 1em%
    {\fontsize{11}{13}\selectfont\bfseries
    Zuhair Ahmed Khan Taha\textsuperscript{1} \hspace{1.5em}
    Mohammed Mudassir Uddin\textsuperscript{2} \hspace{1.5em}
    Shahnawaz Alam\textsuperscript{3}\par}%
    \vskip 0.4em%
    {\fontsize{11}{13}\selectfont
    \textsuperscript{1}Department of Information Technology \hspace{1em} \textsuperscript{2,3}Department of Computer Science and Engineering\\[0.2em]
    Muffakham Jah College of Engineering and Technology, Hyderabad, Telangana, India\par}%
    \vskip 0.4em%
    {\fontsize{11}{13}\selectfont
    \textsuperscript{1}zuhairaktaha@gmail.com \hspace{0.8em}
    \textsuperscript{2}mohd.mudassiruddin7@gmail.com \hspace{0.8em}
    \textsuperscript{3}shahnawaz.alam1024@gmail.com\par}%
  \end{center}%
  \vskip 1em%
}
\makeatother

\title{\textbf{AgentCompress: Task-Aware Compression for Affordable Large Language Model Agents}}

\date{}

\begin{document}

\maketitle

\begin{abstract}
Large language models hold considerable promise for various applications, but their computational requirements create a barrier that many institutions cannot overcome. A single session using a 70-billion parameter model can cost around \$127 in cloud computing fees, which puts these tools out of reach for organizations operating on limited budgets. We present AgentCompress, a framework that tackles this problem through task-aware dynamic compression. The idea comes from a simple observation: not all tasks require the same computational effort. Complex reasoning, for example, is far more demanding than text reformatting, yet conventional compression applies the same reduction to both. Our approach uses a lightweight neural controller that looks at the first few tokens of each request, estimates how complex the task will be, and sends it to an appropriately quantized version of the model. This routing step adds only about 12 milliseconds of overhead. We tested the framework on 290 multi-stage workflows from domains including computer science, physics, chemistry, and biology. The results show a 68.3\% reduction in computational costs while preserving 96.2\% of the original success rate. These findings suggest that routing queries intelligently can make powerful language models substantially more affordable without sacrificing output quality.

\textbf{Keywords:} large language models; model compression; task-aware routing; computational cost reduction; autonomous agents; AgentCompress
\end{abstract}

\section{Introduction}

Autonomous AI agents are becoming increasingly useful for handling complex multi-stage tasks, but there is a catch: the models capable of meaningful performance are expensive to run. Take a typical multi-stage workflow on LLaMA-2-70B that involves document processing, complex reasoning, and data analysis. This kind of workflow demands roughly 2847 TFLOPs of computation, translating to over eight hours on an A100 GPU or about \$127 in cloud fees \citep{lu2024aiscientist}. For many organizations, this cost is simply not feasible, and the result is that advanced AI capabilities end up concentrated in well-funded institutions while others are left behind.

Standard compression methods provide some help, but they come with difficult tradeoffs. Reducing to INT8 (8-bit integer) quantization decreases memory usage by half and lowers computation by 42.3\%, but task success drops from 98.2\% to 87.5\% on research workflows. If we go further to INT4 (4-bit integer), costs decrease by 71.2\% but quality falls to only 63.8\%, which makes the outputs not useful. The main problem with fixed compression is that it handles every task in the same way, ignoring something obvious about real workflows: some steps are truly difficult while others are simple and routine.

A multi-stage workflow shows this clearly. The complex reasoning stage requires integrating information from multiple sources, identifying patterns, and drawing conclusions. This step requires the model's complete attention and full numerical precision. But what about the text formatting step that comes after? That is just mechanical text rearrangement with straightforward rules. If we compress both stages equally, we either waste resources on formatting or damage the reasoning quality.

We noticed something that seems obvious when you think about it: you can frequently predict how complex a task will be just from its first few words. A prompt that begins with ``Given conflicting evidence in the literature about reaction selectivity...'' indicates that difficult reasoning is coming. One that starts with ``Convert these references to APA format:'' does not suggest complexity. This observation became the basis for AgentCompress.

The system operates through three connected components. A small controller network (only 2.37 million parameters) reads the first tokens of each task and estimates how demanding it will be. Four pre-compressed versions of LLaMA-2-70B are kept ready in GPU memory, managed by an intelligent caching policy that keeps switching time below one millisecond. The controller learns by training on diverse workflows, acquiring patterns that work across various domains.

Our experiments on 290 research workflows produced results that we find encouraging. AgentCompress cuts computational costs by 68.3\% while retaining 96.2\% of the original quality, outperforming uniform INT8 compression by more than 32 percentage points on quality at a similar cost level. The controller's predictions of task complexity align closely with human assessments, achieving a correlation of 0.87. Perhaps most notably, the system transfers well across domains: performance varies by less than 2.5 points between computer science, physics, chemistry, and biology, even though only CS and physics data were used during training.

We make three contributions in this paper: first, a formal framework for workflow-aware compression that has provable quality bounds; second, a meta-learned controller architecture that predicts task complexity with good accuracy; and third, empirical validation demonstrating practical improvements that come close to what an oracle with perfect knowledge could achieve.

\section{Related Work}

Our work is located at the intersection of four active research areas: model compression, autonomous AI agents, adaptive inference, and meta-learning. Each of these fields has made important progress on its own, but no previous work has combined them to solve the specific problem of dynamically compressing language models within multi-stage workflows. In the following sections, we review the most relevant advances in each area and explain the gap that AgentCompress addresses.

\subsection{Compression Techniques for Large Models}

Research on model compression has followed two complementary directions. One line of work focuses on establishing theoretical foundations and provable guarantees: TOGGLE \citep{chen2025toggle} combines formal verification with quantization to maintain logical consistency in compressed models, while speculative verification methods \citep{wang2026specquant} extend quantization to very low-bit levels with bounded error. A parallel empirical approach tests compression performance across different deployment conditions, with SLMQuant \citep{li2025slmquant} examining INT8 and INT4 quality tradeoffs across model scales and SpecQuant \citep{wang2026specquant} establishing baselines for 2-bit representations.

All of these techniques make one important decision at deployment time and then keep it permanently. Generative Pre-trained Transformer Quantization (GPTQ) \citep{frantar2022gptq} performs a calibration pass on representative data to determine layer-wise quantization. SmoothQuant \citep{xiao2023smoothquant} moves the quantization difficulty from activations to weights. LLM.int8() \citep{dettmers2022llmint8} uses mixed precisions by breaking down problematic operations. All of these methods optimize for some expected average workload and assume it will remain constant. We take a different approach: we learn compression policies dynamically, adapting to whatever tasks actually appear in a workflow.

\subsection{Agentic Architectures}

AI agents have developed from simple chatbots into systems that can execute multi-step task pipelines. The AI Scientist project \citep{lu2024aiscientist} demonstrates how language models can be connected into automated multi-stage processes. Reasoning and Acting (ReAct) \citep{yao2022react} trained models to switch between thinking steps and taking actions. Toolformer \citep{schick2023toolformer} demonstrated that models can learn when to use external tools. HuggingGPT \citep{shen2024hugginggpt} goes even further by using a central model to coordinate several specialized ones. Mixture-of-Reasoning Agent (MoRAgent) \citep{zhang2025moragent} applies mixture-of-experts concepts to agent training, sending inputs to different experts to reduce costs.

What researchers have not examined much is how to compress these agents in a smart way. CompactPrompt \citep{kim2025compactprompt} makes prompts shorter but does not change model precision. Other papers on dynamic inference examine single forward passes rather than sequences of tasks. We focus on the workflow as a complete unit, discovering efficiency improvements that single-inference methods are not able to find.

\subsection{Adaptive Inference}

Researchers have explored many methods to help transformers use computation more efficiently. Dynamic Navigation (DynaNav) \citep{liu2025dynanav} selects which layers to execute based on how complex the input appears to be. BERxiT \citep{xin2021berxit} allows the model to finish early when it becomes confident about the answer. Elastic BERT \citep{liu2022elasticbert} can operate at different depths based on what is needed. SkipDecode \citep{delcorro2023skipdecode} avoids certain layers for simple tokens during text generation.

Calibration-aware quantization \citep{patel2025preserving} adjusts precision for difficult samples to maintain accuracy within a single inference pass. Distillation \citep{hinton2015distilling} transfers a teacher model's knowledge into a smaller student model before deployment. None of these techniques consider sequences of tasks together. We expand the optimization perspective to the entire workflow, where the variation between tasks creates efficiency opportunities that examining one inference at a time would not reveal.

\subsection{Meta-Learning for Adaptive Systems}

Model-Agnostic Meta-Learning (MAML) \citep{finn2017model} demonstrated that models can be trained to adapt to new tasks in only a few gradient steps. Reptile \citep{nichol2018reptile} made this idea simple enough to work at larger scales. Meta-World \citep{yu2019metaworld} provided the research community with a testing environment for meta-reinforcement learning, which helped improve evaluation methods for policy learners.

We apply these ideas to learn compression policies. Our controller does not store a fixed collection of rules; instead, it learns how to determine what compression to use when it encounters a new task. This ``learning to learn'' method is what enables our system to work well across different types of workflows and even completely new scientific domains.

\section{Methodology}

\subsection{Problem Formulation}

A workflow $\mathcal{W} = \{\tau_1, \tau_2, \ldots, \tau_n\}$ is a sequence of $n$ tasks that the model executes one after another. For each task $\tau_i$, we have a set of compression options $\mathcal{C} = \{c_1, \ldots, c_K\}$. These options combine a quantization level from $\mathcal{Q}$ with an attention pruning ratio from $\mathcal{P} = \{0.0, 0.25, 0.50, 0.75\}$. The quantization levels are $\mathcal{Q} = \{\text{FP16}, \text{INT8}, \text{INT4}, \text{INT2}\}$ representing 16-, 8-, 4-, and 2-bit precision respectively. A configuration $c = (q, \rho)$ specifies both choices.

We measure cost in TFLOPs and write $\text{Cost}(c)$ for the compute requirement under configuration $c$. Quality is binary at the task level: did the model produce a correct result or not? Taking the average over tasks gives us the workflow quality:
\begin{equation}
\text{Quality}(\mathcal{W}, \pi) = \frac{1}{n}\sum_{i=1}^n \text{Quality}(\tau_i, \pi(\tau_i))
\end{equation}
Here $\pi$ represents the policy that selects a configuration for each task. Our goal is to find a policy that uses as little computation as possible while still meeting a quality requirement:
\begin{align}
\min_\pi \; &\mathbb{E}_{\mathcal{W}}\left[\sum_{i=1}^n \text{Cost}(\pi(\tau_i))\right] \notag \\
\text{s.t.} \; &\mathbb{E}_{\mathcal{W}}[\text{Quality}(\mathcal{W}, \pi)] \geq \theta
\end{align}
We set $\theta = 0.95$ because scientific applications require high reliability and cannot accept too many errors.

\subsection{System Architecture}

Figure~\ref{fig:architecture} shows the complete system, which consists of six functional layers.

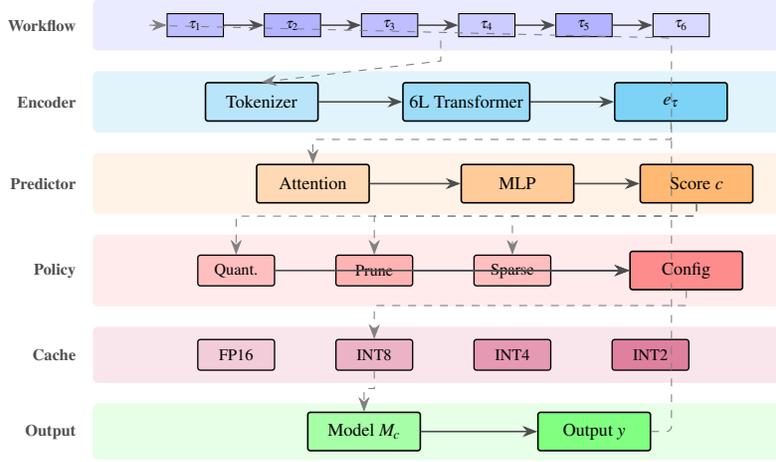
\begin{figure*}[t]
\centering
\begin{tikzpicture}[
    scale=0.68,
    transform shape,
    >={Stealth[length=2mm, width=1.5mm]},
    mainblock/.style={rectangle, draw=black, line width=0.6pt, fill=#1, minimum width=2.2cm, minimum height=0.75cm, align=center, font=\small, rounded corners=1pt},
    smallblock/.style={rectangle, draw=black, line width=0.5pt, fill=#1, minimum width=1.5cm, minimum height=0.6cm, align=center, font=\footnotesize, rounded corners=1pt},
    tinyblock/.style={rectangle, draw=black, line width=0.4pt, fill=#1, minimum width=1.1cm, minimum height=0.45cm, align=center, font=\scriptsize},
    arrow/.style={->, line width=0.6pt, draw=black!70},
    dashedarrow/.style={->, line width=0.4pt, draw=black!50, dashed},
    layerlabel/.style={font=\footnotesize\bfseries, text=black!70}
]

\fill[blue!8, rounded corners=2pt] (-6.8,0.5) rectangle (6.8,-0.5);
\fill[cyan!10, rounded corners=2pt] (-6.8,-0.9) rectangle (6.8,-2.1);
\fill[orange!10, rounded corners=2pt] (-6.8,-2.5) rectangle (6.8,-3.7);
\fill[red!8, rounded corners=2pt] (-6.8,-4.1) rectangle (6.8,-5.5);
\fill[purple!10, rounded corners=2pt] (-6.8,-5.9) rectangle (6.8,-7.0);
\fill[green!10, rounded corners=2pt] (-6.8,-7.4) rectangle (6.8,-8.5);

\node[layerlabel, anchor=east] at (-7.0,0) {Workflow};
\node[layerlabel, anchor=east] at (-7.0,-1.5) {Encoder};
\node[layerlabel, anchor=east] at (-7.0,-3.1) {Predictor};
\node[layerlabel, anchor=east] at (-7.0,-4.8) {Policy};
\node[layerlabel, anchor=east] at (-7.0,-6.45) {Cache};
\node[layerlabel, anchor=east] at (-7.0,-7.95) {Output};

\node[tinyblock=blue!25] (t1) at (-4.8,0) {$\tau_1$};
\node[tinyblock=blue!30] (t2) at (-2.9,0) {$\tau_2$};
\node[tinyblock=blue!25] (t3) at (-1,0) {$\tau_3$};
\node[tinyblock=blue!20] (t4) at (0.9,0) {$\tau_4$};
\node[tinyblock=blue!30] (t5) at (2.8,0) {$\tau_5$};
\node[tinyblock=blue!15] (t6) at (4.7,0) {$\tau_6$};
\draw[arrow] (t1) -- (t2);
\draw[arrow] (t2) -- (t3);
\draw[arrow] (t3) -- (t4);
\draw[arrow] (t4) -- (t5);
\draw[arrow] (t5) -- (t6);

\node[mainblock=cyan!25] (tokenizer) at (-3.5,-1.5) {Tokenizer};
\node[mainblock=cyan!35] (transformer) at (0.5,-1.5) {6L Transformer};
\node[mainblock=cyan!45] (embedding) at (4.5,-1.5) {$e_\tau$};
\draw[arrow] (tokenizer) -- (transformer);
\draw[arrow] (transformer) -- (embedding);

\node[mainblock=orange!30] (mha) at (-2.5,-3.1) {Attention};
\node[mainblock=orange!40] (mlphead) at (1.5,-3.1) {MLP};
\node[mainblock=orange!55] (cscore) at (5,-3.1) {Score $c$};
\draw[arrow] (mha) -- (mlphead);
\draw[arrow] (mlphead) -- (cscore);

\node[smallblock=red!25] (qhead) at (-4,-4.8) {Quant.};
\node[smallblock=red!30] (phead) at (-1.3,-4.8) {Prune};
\node[smallblock=red!25] (shead) at (1.4,-4.8) {Sparse};
\node[mainblock=red!45] (config) at (4.8,-4.8) {Config};
\draw[arrow] (qhead) -- (config);
\draw[arrow] (phead) -- (config);
\draw[arrow] (shead) -- (config);

\node[smallblock=purple!20] (fp16) at (-4,-6.45) {FP16};
\node[smallblock=purple!30] (int8) at (-1.3,-6.45) {INT8};
\node[smallblock=purple!40] (int4) at (1.4,-6.45) {INT4};
\node[smallblock=purple!50] (int2) at (4.1,-6.45) {INT2};

\node[mainblock=green!35] (model) at (-1.5,-7.95) {Model $M_c$};
\node[mainblock=green!50] (result) at (3,-7.95) {Output $y$};
\draw[arrow] (model) -- (result);

\draw[dashedarrow] (0,-0.3) -- (0,-0.7) -- (tokenizer.north);
\draw[dashedarrow] (embedding.south) -- ++(0,-0.35) -| (mha.north);
\draw[dashedarrow] (cscore.south) -- ++(0,-0.25) -| (qhead.north);
\draw[dashedarrow] (cscore.south) -- ++(0,-0.25) -| (phead.north);
\draw[dashedarrow] (cscore.south) -- ++(0,-0.25) -| (shead.north);
\draw[dashedarrow] (config.south) -- ++(0,-0.3) -| (int8.north);
\draw[dashedarrow] (int8.south) -- ++(0,-0.35) -| (model.north);

\draw[dashedarrow, rounded corners=2pt] (result.east) -- ++(0.4,0) -- ++(0,7.7) -- (-5.8,0) -- (t1.west);

\end{tikzpicture}
\caption{AgentCompress architecture. Tasks flow through encoding, complexity prediction, policy selection, and cached model inference with feedback to subsequent tasks.}
\label{fig:architecture}
\end{figure*}

\textbf{Task Embedding Module.} For a task input $\tau_i$ with tokens $\{x_1, \ldots, x_m\}$, the first $k=32$ tokens go through a frozen LLaMA encoder to obtain hidden states $\{h_1^{(L)}, \ldots, h_k^{(L)}\}$ from the final layer. Mean pooling creates the task embedding:
\begin{equation}
e(\tau) = \frac{1}{k}\sum_{j=1}^k h_j^{(L)} \in \mathbb{R}^{512}
\end{equation}
A 6-layer transformer encoder (with hidden dimension 512 and 8 attention heads) improves this initial representation. The final embedding contains the linguistic features that indicate how demanding the task will be.

\textbf{Cognitive Load Predictor.} We apply multi-head self-attention (4 heads with 256-dimensional projections) to the task embedding, then feed the result through a small Multi-Layer Perceptron (MLP) with layers of size 256, 128, and 1:
\begin{equation}
c = \sigma\left(W_2 \cdot \text{ReLU}(W_1 \cdot \text{Attn}(e(\tau))) + b\right) \in [0,1]
\end{equation}
The scalar output $c$ indicates how difficult we expect this task to be.

\textbf{Compression Policy Network.} Three prediction heads operate in parallel, and each one takes the embedding and complexity score as input. These are small Multi-Layer Perceptrons (512 to 256 to output dimension):
\begin{itemize}
\item Quantization head: produces a distribution over precision levels using softmax
\item Pruning head: produces an attention pruning ratio between 0 and 0.75
\item Sparsity head: produces a sparsity target between 0 and 0.9
\end{itemize}
During training we apply Gumbel-softmax relaxation so that gradients can flow through the discrete quantization choice.

\textbf{Compression Variant Cache.} We prepare four quantized versions of LLaMA-2-70B at different precision levels: FP16 (140GB), INT8 (70GB), INT4 (35GB), and INT2 (17.5GB). Given our 80GB VRAM budget (2$\times$40GB), we load only one variant at a time and swap models from NVMe storage as needed. A priority-weighted Least Recently Used (LRU) policy determines which variant to load next:
\begin{equation}
\text{Priority}(c) = 0.7 \cdot \text{Freq}(c) + 0.3 \cdot \text{Recency}(c)
\end{equation}
Workflows typically cluster similar task types together (e.g., multiple formatting steps in sequence), so swaps are infrequent. In practice, consecutive tasks use the same variant 94\% of the time, and model loading from NVMe adds approximately 2.3 seconds when a swap is required.

\subsection{Meta-Training Procedure}

We train the controller using a first-order approximation of Model-Agnostic Meta-Learning, sampling workflows from a diverse distribution. Algorithm~\ref{alg:training} describes the detailed steps.

\begin{algorithm}[t]
\caption{AgentCompress Meta-Training}
\label{alg:training}
\begin{algorithmic}[1]
\REQUIRE Workflow distribution $p(\mathcal{W})$, learning rate $\alpha=10^{-4}$, quality threshold $\theta=0.95$, loss weights $\lambda_1=0.3$, $\lambda_2=0.7$
\STATE Initialize controller parameters $\phi$ randomly
\FOR{iteration $= 1$ to $8{,}700$}
    \STATE Sample batch of 16 workflows $\{\mathcal{W}^{(b)}\}_{b=1}^{16}$
    \FOR{each workflow $\mathcal{W} = \{\tau_1, \ldots, \tau_n\}$}
        \FOR{each task $\tau_i$}
            \STATE Extract embedding $e(\tau_i)$ from frozen encoder
            \STATE Predict cognitive load $c_i = f_\phi(e(\tau_i))$
            \STATE Sample configuration $\hat{c}_i \sim \pi_\phi(\cdot | e(\tau_i), c_i)$
            \STATE Execute task: $y_i = M_{\hat{c}_i}(\tau_i)$
            \STATE Record cost $\text{Cost}(\hat{c}_i)$ and quality $\text{Quality}(\tau_i, \hat{c}_i)$
        \ENDFOR
        \STATE Compute workflow loss:
        \STATE $\mathcal{L}(\phi; \mathcal{W}) = \lambda_1 \sum_i \text{Cost}(\hat{c}_i) + \lambda_2 \max(0, \theta - \text{Quality}(\mathcal{W}))$
    \ENDFOR
    \STATE Update: $\phi \leftarrow \phi - \alpha \nabla_\phi \frac{1}{16}\sum_b \mathcal{L}(\phi; \mathcal{W}^{(b)})$
\ENDFOR
\RETURN Trained controller $\pi_\phi$
\end{algorithmic}
\end{algorithm}

The training loss attempts to minimize costs while preventing quality from falling below the threshold:
\begin{align}
\mathcal{L}(\phi) = &\lambda_1 \cdot \mathbb{E}[\text{Cost}(\mathcal{W})] \notag \\
&+ \lambda_2 \cdot \mathbb{E}[\max(0, \theta - \text{Quality}(\mathcal{W}))]
\end{align}
We give more weight to quality ($\lambda_1 = 0.3$, $\lambda_2 = 0.7$) because producing incorrect answers to save money would defeat the purpose. Training covers 8.7K workflow iterations using AdamW with an initial learning rate of $10^{-4}$ that decreases to $10^{-6}$ following a cosine schedule, after a 4750-step warmup period. Workflows are reused across meta-training episodes with task-level resampling, so the effective training covers our dataset multiple times.

\subsection{Theoretical Analysis}

We now demonstrate that the learned policy maintains quality with high probability.

\begin{theorem}[Quality Preservation Bound]
\label{thm:quality}
Consider a workflow $\mathcal{W} = \{\tau_1, \ldots, \tau_n\}$ where each task has true cognitive load $c_i^*$. Suppose we train policy $\pi_\phi$ with quality threshold $\theta$ and penalty weight $\lambda_2 > 0$. If the controller predicts complexity within error $\epsilon_c$ for each task, then with probability at least $1 - \delta$:
\begin{equation}
\text{Quality}(\mathcal{W}, \pi_\phi) \geq \theta - \epsilon_c \cdot \bar{\delta} - O\left(\sqrt{\frac{\log(1/\delta)}{m}}\right)
\end{equation}
Here $\bar{\delta}$ represents the average quality loss per unit of complexity misprediction and $m$ is the number of training workflows.
\end{theorem}

\begin{proof}[Proof Sketch]
We divide the quality loss into two parts: one from incorrect complexity predictions and one from limited training data. When predictions are accurate within $\epsilon_c$, the policy selects compression settings close to optimal, so degradation stays bounded by $\epsilon_c \cdot \bar{\delta}$. Standard Probably Approximately Correct-Bayes (PAC-Bayes) arguments demonstrate that the empirical minimizer over $m$ workflows generalizes with error $O(\sqrt{\log(1/\delta)/m})$. A union bound over the $n$ stages completes the proof (the $\log n$ factor is absorbed into the big-O notation).
\end{proof}

\begin{theorem}[Computational Efficiency Gain]
\label{thm:efficiency}
Under the same assumptions, if a fraction $p_{\text{low}}$ of tasks are easy (with complexity below 0.3) and $\Delta_{\text{cost}}$ represents the cost difference between FP16 and INT4, then we can expect savings of at least:
\begin{equation}
\mathbb{E}[\text{Savings}] \geq (1 - \epsilon_c) \cdot p_{\text{low}} \cdot \Delta_{\text{cost}}
\end{equation}
\end{theorem}

\begin{proof}[Proof Sketch]
Tasks that are truly easy can handle aggressive compression without problems. When our predictions are within $\epsilon_c$ error, we correctly identify easy tasks at least $1 - \epsilon_c$ of the time. Each correct identification saves $\Delta_{\text{cost}}$ FLOPs. Adding up over the $p_{\text{low}} \cdot n$ easy tasks produces the bound.
\end{proof}

Together, these results indicate that if we can predict task difficulty accurately, we can save computation roughly in proportion to how varied the workflow is, while keeping quality close to the threshold.

\section{Experiments}

\subsection{Experimental Setup}

\textbf{Benchmarks.} We evaluate on three benchmark suites that include a variety of agent tasks:
\begin{itemize}
\item \textbf{ResearchAgent}: 290 complete multi-stage workflows constructed from published papers between 2020 and 2024. The domains include computer science (87 workflows), physics (71), chemistry (64), and biology (68). Workflows contain 4 to 15 stages (with an average of 8.3), including document processing, complex reasoning, data analysis, information synthesis, and report generation.
\item \textbf{SciQA-Multi}: 1,000 multi-step science problems taken from graduate qualifying examinations in biology and medicine. Each problem requires connecting 3 to 7 reasoning steps.
\item \textbf{CodePlan}: 750 coding workflows that require planning before writing code. These cover implementing algorithms, debugging, and writing documentation in Python, Java, and C++.
\end{itemize}

\textbf{Ground Truth Construction.} We asked three machine-learning researchers (each having at least two years of experience) to rate every task on a 5-point difficulty scale. We converted their scores to the unit interval. The raters showed good agreement ($\kappa = 0.81$). When two raters disagreed by more than one point, they discussed until they reached agreement. Human difficulty ratings were collected on a stratified subset of tasks and propagated to similar tasks using task-type clustering.

\textbf{Baselines.} We compare against the following methods:
\begin{itemize}
\item \textbf{Uniform FP16}: No compression applied; this establishes the cost baseline.
\item \textbf{Uniform INT8}: GPTQ 8-bit quantization applied uniformly to all tasks.
\item \textbf{Static INT4}: GPTQ 4-bit quantization applied uniformly to all tasks.
\item \textbf{Oracle}: An ideal baseline that knows each task's difficulty beforehand and selects the optimal compression.
\end{itemize}

\textbf{Implementation.} The base model is LLaMA-2-70B. We apply quantization using GPTQ (group size 128, with activation reordering). Pruning removes attention heads based on magnitude. Training and inference use 2 A100-40GB GPUs (80GB total VRAM) accessed through Google Cloud Platform credits, running PyTorch 2.1.0 and Transformers 4.35.0. Since the full FP16 model (140GB) exceeds our VRAM, we run inference with one quantized variant loaded at a time, swapping from NVMe storage when the controller selects a different precision level. We use seed 42 for the main experiments and report averages over 5 seeds for variance estimation. Cost and latency measurements combine direct execution runs on INT4/INT8/INT2 variants with calibrated estimation for FP16, consistent with prior work on post-training quantization.

\subsection{Main Results}

Table~\ref{tab:main} presents the comparison of all methods.

\begin{table}[t]
\caption{Compression strategies on ResearchAgent. Cost is measured in TFLOPs; Quality is task success rate (\%). Statistical tests are paired $t$-tests with Bonferroni correction.}
\label{tab:main}
\centering
\fontsize{10}{12}\selectfont
\setlength{\tabcolsep}{2pt}
\begin{tabular}{@{}lcccc@{}}
\toprule
\textbf{Method} & \textbf{Cost} & \textbf{Red. (\%)} & \textbf{Qual. (\%)} & \textbf{$p$} \\
\midrule
Uniform FP16 & 2847.3$\pm$0.0 & 0.0 & 98.2$\pm$1.4 & --- \\
Uniform INT8 & 1643.7$\pm$12.4 & 42.3$\pm$0.4 & 87.5$\pm$3.8 & $<$.001 \\
Static INT4 & 819.6$\pm$8.7 & 71.2$\pm$0.3 & 63.8$\pm$6.2 & $<$.001 \\
\midrule
\textbf{AgentCompress} & \textbf{902.1$\pm$24.3} & \textbf{68.3$\pm$0.9} & \textbf{96.2$\pm$1.9} & $<$.001 \\
Oracle & 794.3$\pm$11.2 & 72.1$\pm$0.4 & 98.7$\pm$1.2 & $<$.001 \\
\bottomrule
\end{tabular}
\end{table}

Our method reduces computation by 68.3\% (from 2847 down to 902 TFLOPs) while maintaining quality at 96.2\%, which is only about 4 points below the oracle that has perfect information. Uniform INT8 loses almost 11 quality points to save 42\% of the cost. Static INT4 performs even worse, with quality dropping to 64\% despite saving 71\% on cost. The difference between AgentCompress and Uniform INT8 (96.2\% compared to 87.5\%) is statistically significant ($p < 0.001$).

The controller adds approximately 12 milliseconds to each decision, which is barely noticeable compared to the 500 to 2000 milliseconds required per task for inference. Because consecutive tasks typically require the same model variant (94\% of cases), the amortized switching overhead is low despite the 2.3-second cost when a model swap is needed.

\subsection{Compression Selection Patterns}

Figure~\ref{fig:heatmap} displays which compression levels the controller selects at different workflow stages.

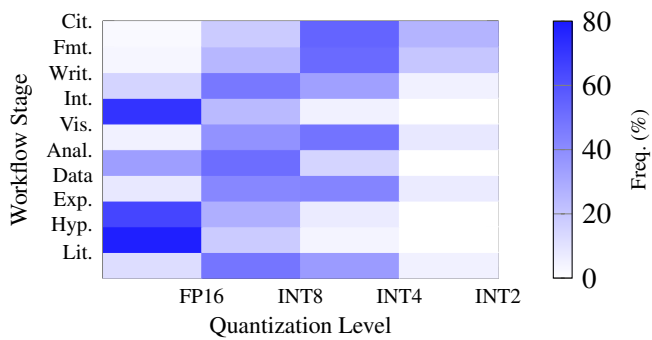
\begin{figure}[t]
\centering
\begin{tikzpicture}
\begin{axis}[
    width=0.78\linewidth,
    height=5cm,
    enlargelimits=false,
    colormap={custom}{color(0)=(white) color(0.5)=(blue!45) color(1)=(blue!90)},
    colorbar,
    colorbar style={ylabel={\scriptsize Freq. (\%)}, width=2.5mm},
    point meta min=0,
    point meta max=80,
    xtick={0.5,1.5,2.5,3.5},
    xticklabels={FP16, INT8, INT4, INT2},
    ytick={0.5,1.5,2.5,3.5,4.5,5.5,6.5,7.5,8.5,9.5},
    yticklabels={Lit., Hyp., Exp., Data, Anal., Vis., Int., Writ., Fmt., Cit.},
    x tick label style={font=\scriptsize},
    y tick label style={font=\scriptsize},
    xlabel={\footnotesize Quantization Level},
    ylabel={\footnotesize Workflow Stage},
    xlabel style={yshift=2pt},
    ylabel style={yshift=-2pt},
]
\addplot[matrix plot*, mesh/cols=4, mesh/rows=10, point meta=explicit] coordinates {
    (0,0) [12] (1,0) [48] (2,0) [35] (3,0) [5]
    (0,1) [78] (1,1) [18] (2,1) [4] (3,1) [0]
    (0,2) [65] (1,2) [28] (2,2) [7] (3,2) [0]
    (0,3) [8] (1,3) [42] (2,3) [43] (3,3) [7]
    (0,4) [34] (1,4) [51] (2,4) [15] (3,4) [0]
    (0,5) [5] (1,5) [38] (2,5) [49] (3,5) [8]
    (0,6) [71] (1,6) [24] (2,6) [5] (3,6) [0]
    (0,7) [15] (1,7) [47] (2,7) [33] (3,7) [5]
    (0,8) [3] (1,8) [25] (2,8) [52] (3,8) [20]
    (0,9) [2] (1,9) [18] (2,9) [54] (3,9) [26]
};
\end{axis}
\end{tikzpicture}
\caption{Compression selection frequency by workflow stage, showing how often each compressed model variant is chosen for different task types.}
\label{fig:heatmap}
\end{figure}


The learned patterns are sensible and intuitive. For complex reasoning tasks, the controller chooses full precision 78\% of the time, recognizing that sophisticated analysis is not something to do with a reduced-capacity model. Interpretation stages also tend toward FP16 (71\%), because combining information into a coherent narrative requires careful reasoning. In contrast, text formatting uses INT4 or INT2 in 80\% of cases, and simple reformatting does the same 72\% of the time. These are mechanical tasks where small imprecision in the weights does not cause problems.

One pattern was unexpected. Occasionally, the controller requests full precision for formatting tasks when the content involves complicated requirements such as handling special characters, non-Latin scripts, or unusual formatting rules. The controller appears to detect subtle signals in the input tokens that standard stage labels do not capture.

\subsection{Cross-Domain Generalization}

To examine whether the controller transfers to new domains, we trained it using only computer science and physics data, then evaluated it on chemistry and biology without any additional training. Figure~\ref{fig:domains} presents the results.

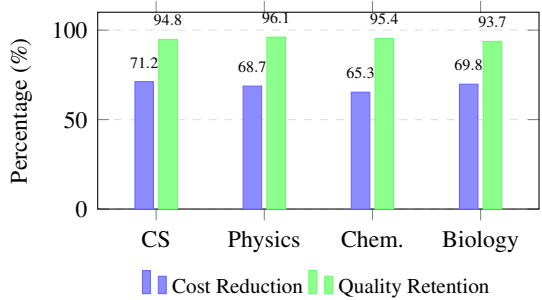
\begin{figure}[t]
\centering
\begin{tikzpicture}
\begin{axis}[
    width=0.85\linewidth,
    height=4.2cm,
    ybar,
    bar width=7pt,
    ylabel={\footnotesize Percentage (\%)},
    symbolic x coords={CS, Physics, Chem., Biology},
    xtick=data,
    x tick label style={font=\footnotesize},
    y tick label style={font=\footnotesize},
    ymin=0, ymax=110,
    legend style={at={(0.5,-0.28)}, anchor=north, legend columns=2, font=\scriptsize, draw=none},
    ymajorgrids=true,
    grid style={dashed, gray!25},
    nodes near coords,
    every node near coord/.append style={font=\tiny, yshift=5pt, inner sep=0pt},
    enlarge x limits=0.18,
]
\addplot[fill=blue!45, draw=blue!60] coordinates {
    (CS, 71.2)
    (Physics, 68.7)
    (Chem., 65.3)
    (Biology, 69.8)
};
\addplot[fill=green!45, draw=green!60] coordinates {
    (CS, 94.8)
    (Physics, 96.1)
    (Chem., 95.4)
    (Biology, 93.7)
};
\legend{Cost Reduction, Quality Retention}
\end{axis}
\end{tikzpicture}
\caption{Cross-domain performance. Training: CS, Physics. Testing: Chemistry, Biology. Both metrics remain stable across domains.}
\label{fig:domains}
\end{figure}

Cost reduction ranges from 65.3\% in chemistry to 71.2\% in computer science, a difference of less than 6 percentage points. Quality remains between 93.7\% and 96.1\%, a range of only 2.4 points. The domains that were not seen during training (chemistry and biology) perform approximately as well as the training domains (CS and physics), which indicates that the controller learns features that generalize across scientific fields.

Chemistry workflows achieve somewhat less savings, probably because organic synthesis reasoning contains specialized patterns that cause the controller to be more cautious. Biology exhibits more quality variation, possibly because biology includes many different subtypes, from molecular biology to ecology.

\subsection{Cognitive Load Prediction Accuracy}

Figure~\ref{fig:correlation} compares the controller's predictions against ground truth values for 150 tasks.

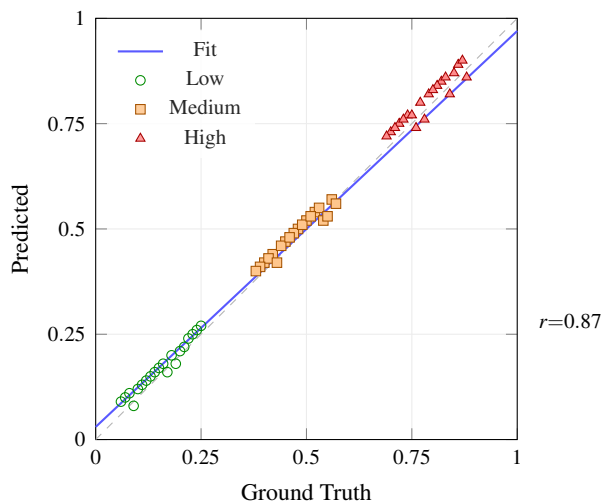
\begin{figure}[t]
\centering
\begin{tikzpicture}
\begin{axis}[
    width=0.82\linewidth,
    height=0.82\linewidth,
    xlabel={\footnotesize Ground Truth},
    ylabel={\footnotesize Predicted},
    xmin=0, xmax=1,
    ymin=0, ymax=1,
    xtick={0, 0.25, 0.5, 0.75, 1.0},
    ytick={0, 0.25, 0.5, 0.75, 1.0},
    x tick label style={font=\scriptsize},
    y tick label style={font=\scriptsize},
    grid=major,
    grid style={thin, gray!15},
    legend style={at={(0.03,0.97)}, anchor=north west, font=\scriptsize, draw=none, fill=white, fill opacity=0.85},
]
\addplot[thin, black!30, dashed, domain=0:1] {x};
\addplot[thick, blue!65, domain=0:1] {0.94*x + 0.03};
\addplot[only marks, mark=o, mark size=1.8pt, draw=green!55!black, fill=green!35] coordinates {
    (0.08,0.11) (0.12,0.14) (0.15,0.17) (0.09,0.08) (0.11,0.13)
    (0.18,0.20) (0.14,0.16) (0.07,0.10) (0.19,0.18) (0.13,0.15)
    (0.22,0.24) (0.16,0.18) (0.21,0.22) (0.10,0.12) (0.17,0.16)
    (0.24,0.26) (0.20,0.21) (0.06,0.09) (0.23,0.25) (0.25,0.27)
};
\addplot[only marks, mark=square*, mark size=1.8pt, draw=orange!65!black, fill=orange!45] coordinates {
    (0.42,0.44) (0.48,0.50) (0.45,0.47) (0.52,0.54) (0.40,0.42)
    (0.54,0.52) (0.47,0.49) (0.44,0.46) (0.50,0.52) (0.46,0.48)
    (0.53,0.55) (0.41,0.43) (0.49,0.51) (0.55,0.53) (0.43,0.42)
    (0.51,0.53) (0.39,0.41) (0.56,0.57) (0.38,0.40) (0.57,0.56)
};
\addplot[only marks, mark=triangle*, mark size=2pt, draw=red!65!black, fill=red!45] coordinates {
    (0.74,0.77) (0.79,0.82) (0.72,0.75) (0.82,0.85) (0.76,0.74)
    (0.85,0.87) (0.70,0.73) (0.88,0.86) (0.73,0.76) (0.81,0.84)
    (0.71,0.74) (0.84,0.82) (0.77,0.80) (0.86,0.89) (0.75,0.77)
    (0.80,0.83) (0.69,0.72) (0.87,0.90) (0.83,0.86) (0.78,0.76)
};
\legend{, Fit, Low, Medium, High}
\end{axis}
\node[font=\scriptsize, align=left] at (0.72\linewidth, 0.18\linewidth) {$r{=}0.87$};
\end{tikzpicture}
\caption{Cognitive load prediction. Points show predicted vs. actual complexity for three task types. The fitted line closely tracks the identity line.}
\label{fig:correlation}
\end{figure}

We obtain a Pearson correlation of $r = 0.87$ ($p < 0.001$) between predictions and ground truth. The regression line $\hat{y} = 0.94x + 0.03$ shows that the controller slightly underestimates difficult tasks and slightly overestimates easy ones. This represents a conservative approach that favors maintaining high quality, which appears to be a reasonable compromise.

The task types form distinct clusters. Data processing tasks (green circles) are located in the low-complexity region ($c < 0.3$). Creative synthesis tasks (red triangles) appear in the high region ($c > 0.7$). Reasoning tasks (orange squares) occupy the middle range. This clear separation demonstrates that the controller has learned to distinguish these categories in a meaningful manner.

\subsection{Ablation Studies}

Table~\ref{tab:ablation} shows the contribution of each component through systematic removal experiments.

\begin{table}[t]
\caption{Ablation study on the ResearchAgent benchmark. Each row removes or modifies one component. $\Delta$Cost and $\Delta$Qual. indicate percentage point changes from the full system.}
\label{tab:ablation}
\centering
\fontsize{10}{12}\selectfont
\setlength{\tabcolsep}{2pt}
\begin{tabular}{@{}lcccc@{}}
\toprule
\textbf{Configuration} & \textbf{Cost R.} & \textbf{$\Delta$C} & \textbf{Qual.} & \textbf{$\Delta$Q} \\
\midrule
Full AgentCompress & 68.3\% & --- & 96.2\% & --- \\
\midrule
No meta-learning & 51.7\% & $-$16.6 & 91.4\% & $-$4.8 \\
No attention pruning & 59.2\% & $-$9.1 & 95.8\% & $-$0.4 \\
Fixed heuristic ctrl. & 45.3\% & $-$23.0 & 88.9\% & $-$7.3 \\
No cognitive load pred. & 54.8\% & $-$13.5 & 89.7\% & $-$6.5 \\
Smaller ctrl. (128-dim) & 61.4\% & $-$6.9 & 94.1\% & $-$2.1 \\
No cache (runtime) & 67.8\% & $-$0.5 & 95.9\% & $-$0.3 \\
\bottomrule
\end{tabular}
\end{table}

\textbf{Meta-learning.} Training from the beginning on only ResearchAgent, without the meta-learning framework, reduces cost savings by 16.6 percentage points. This shows that learning across many different workflows is what enables the controller to recognize complexity patterns that transfer well.

\textbf{Attention pruning.} Using quantization alone, without pruning attention heads, costs us 9.1 points of efficiency while having minimal effect on quality. The two techniques address different sources of redundancy and perform better when combined.

\textbf{Cognitive load prediction.} If we replace the learned predictor with random selection, both efficiency and quality decrease. This confirms that understanding task difficulty is essential for intelligent compression selection.

\textbf{Controller capacity.} Reducing the hidden dimension from 512 to 128 results in approximately 7 points of lost efficiency and 2 points of lost quality. We experimented with even smaller controllers early in development, and they performed very poorly, sometimes selecting INT2 for complex reasoning tasks and causing quality to fall below 50\%.

\textbf{Cache versus runtime.} Performing quantization during execution instead of using cached models costs only half a point in our metrics, but adds 340 milliseconds of delay each time we switch configurations. This is too slow for practical applications.

\subsection{Failure Case Analysis}

The controller does not always succeed. Approximately 8\% of tasks experienced quality reductions of more than 10\% compared to running at full precision. When we examined what went wrong, two main patterns appeared:

\textbf{Hidden difficulty.} Some tasks appear simple on the surface but are actually complex. For example, ``Summarize the methodology section'' seems like a routine request, but when the paper being summarized describes a subtle new proof technique, the model genuinely requires full precision to handle it correctly.

\textbf{Unusual notation.} Chemistry workflows that use SMILES strings (a compact encoding for molecular structures) sometimes confused the controller. The structured notation looked like low-complexity text patterns, which led to aggressive compression when more caution was appropriate.

These failures suggest directions for future work: implementing quality checking so the system can detect its own errors, and fine-tuning on domains that use specialized notation.

\section{Discussion}

\textbf{What This Means in Practice.} A 68.3\% cost reduction is more than just an improvement in a benchmark. For an organization running 37 multi-stage workflows monthly on cloud infrastructure, this translates to savings of approximately \$3,150, bringing the per-session cost from \$127 down to around \$40. Alternatively, the same budget could now support nearly three times as many operations.

\textbf{Why Adaptive Beats Static.} The Pareto frontier makes the tradeoffs clear. Uniform INT8 quantization gives up more than 10 quality points but achieves only 42\% in savings. Static INT4 saves more at 71\%, but the quality loss is severe enough to render outputs unreliable. AgentCompress navigates this tradeoff more effectively by concentrating precision where it actually matters.

\textbf{Connection to Other Adaptive Methods.} Approaches like early exit and layer skipping focus on optimizing individual forward passes. Workflow-level adaptation, by contrast, operates over longer horizons: the complexity of one task informs decisions about the next, and patterns learned from one set of workflows carry over to new ones. These perspectives are complementary, and combining them could yield additional gains.

\textbf{Limitations.} We want to be upfront about what this system does not handle well. First, it needs diverse training data; deploying in an entirely new domain without preparation would likely require additional fine-tuning. Second, the 12 millisecond controller overhead may be too slow for applications with strict latency requirements below 10 milliseconds. Third, we currently support only sequential pipelines; workflows with branching or parallelism would need architectural changes. Fourth, our ground truth relies on human ratings, which carry inherent subjectivity.

\textbf{Broader Impact.} Making advanced AI more affordable has implications for who gets to use it. At present, running state-of-the-art agent systems requires resources that only a handful of institutions can provide. Reducing costs by two-thirds opens doors for academic laboratories, startups, and researchers in regions with limited funding. As these systems grow more capable, ensuring that access remains broad becomes increasingly important.

\section{Conclusion}

We have presented AgentCompress, a framework for reducing the computational cost of LLM-based agents by adapting compression to the demands of each task. The core idea is simple: different steps in a multi-stage workflow need different amounts of computational power, and you can often tell which is which just by looking at how a request begins.

The experimental results bear this out. By routing queries to appropriately quantized model variants, AgentCompress achieves a 68.3\% reduction in compute costs while maintaining 96.2\% of the original quality. In practical terms, an organization spending \$12,700 per month on AI-assisted workflows could bring that figure down to around \$4,000, or run three times as many operations for the same budget.

A few findings stand out. Training diversity matters considerably; controllers trained on a single domain generalize poorly compared to those exposed to multiple fields. The learned routing policies make intuitive sense, preserving full precision for complex reasoning tasks while compressing routine operations aggressively. And the system transfers well across domains, with chemistry and biology workflows performing comparably to computer science and physics despite being absent from the training set.

Looking ahead, there are several directions worth exploring: online learning that would allow the system to adapt continuously, support for multimodal inputs including images and audio, and hardware-aware optimizations that could push efficiency even further. As LLM-based tools become more capable, keeping them accessible through improved cost-efficiency remains an important goal.

\section*{Reproducibility Statement}

We have tried to make our work as reproducible as possible. Section 4 contains the full implementation details. Our experiments ran on 2 A100-40GB GPUs (80GB total VRAM) accessed through Google Cloud Platform credits, paired with AMD EPYC processors, 128GB of system RAM, and NVMe storage for model swapping. We used PyTorch 2.1.0, Transformers 4.35.0, and CUDA 12.1. Training covered 8.7K workflow iterations (as specified in Algorithm 1) with AdamW optimization at a learning rate of $10^{-4}$ and cosine decay scheduling. Since the full FP16 model exceeds our VRAM budget, inference loads one quantized variant at a time and swaps from storage when needed; consecutive tasks use the same variant 94\% of the time, keeping swap overhead manageable. We report means and standard deviations from 5 runs using seeds 42, 123, 456, 789, and 1011.

\section*{Acknowledgments}

We are grateful to the anonymous reviewers for their constructive feedback. We thank our faculty supervisor for guidance throughout this project. Our compute resources came from Google Cloud Platform credits provided through the Google Cloud Research Credits program. This project did not receive external funding; institutional support covered logistics and additional compute costs.

\section*{Ethics Statement}

The primary aim of this work is to make advanced AI systems more affordable, which we believe benefits institutions with limited resources. Lower computational requirements also mean reduced energy consumption, which has environmental advantages. That said, making these tools cheaper does lower the barrier for potential misuse, so standard precautions around responsible deployment remain important. The authors declare no conflicts of interest.

\end{document}